\documentclass[letterpaper, 10 pt, journal, twoside]{IEEEtran}
\usepackage[T1]{fontenc}
\usepackage[latin9]{inputenc}
\usepackage{color}
\usepackage{verbatim}
\usepackage{float}
\usepackage{amsmath}

\usepackage{amsthm}
\usepackage{amssymb}
\usepackage{stmaryrd}
\usepackage{stackrel}
\usepackage{graphicx}
\usepackage{wasysym}
\usepackage{mathtools}
\usepackage{hyperref}
\usepackage{url}   
\usepackage{booktabs} 
\usepackage{nicefrac} 
\usepackage{authblk}
\usepackage{multirow}
\usepackage{xcolor}
\usepackage{microtype}
\usepackage{paralist}

\makeatletter

\floatstyle{ruled}
\newfloat{algorithm}{tbp}{loa}
\providecommand{\algorithmname}{Algorithm}
\floatname{algorithm}{\protect\algorithmname}

\theoremstyle{plain}
\newtheorem{thm}{\protect\theoremname}
\theoremstyle{definition}
\newtheorem{defn}{\protect\definitionname}
\theoremstyle{definition}
\newtheorem{problem}{\protect\problemname}
\theoremstyle{plain}
\newtheorem{lem}{\protect\lemmaname}
\theoremstyle{definition}
\newtheorem{example}{\protect\examplename}
\theoremstyle{remark}
\newtheorem{rem}{\protect\remarkname}
\theoremstyle{plain}

\usepackage{cite}
\usepackage{algpseudocode}
\usepackage{setspace}
\IEEEoverridecommandlockouts

\makeatother

\usepackage{babel}
\providecommand{\corollaryname}{Corollary}
\providecommand{\definitionname}{Definition}
\providecommand{\examplename}{Example}
\providecommand{\lemmaname}{Lemma}
\providecommand{\problemname}{Problem}
\providecommand{\remarkname}{Remark}
\providecommand{\theoremname}{Theorem}

\newcommand{\UntilOp}{\mathcal{U}}
\newcommand{\Eventually}{\diamondsuit}
\newcommand{\Always}{\square}
\newcommand{\Next}{\Circle}
\newcommand{\AP}{{AP}}

\usepackage{amsmath}               
  {
      \theoremstyle{plain}
      
      \theoremstyle{plain}
      
  }
  
\SetSymbolFont{stmry}{bold}{U}{stmry}{m}{n}
\def\BibTeX{{\rm B\kern-.05em{\sc i\kern-.025em b}\kern-.08em
		T\kern-.1667em\lower.7ex\hbox{E}\kern-.125emX}}
		
\begin{document}

\title{Overcoming Exploration: Deep Reinforcement Learning for Continuous Control in Cluttered Environments from Temporal Logic Specifications

% Overcoming Exploration: Deep Reinforcement Learning for Continuous Navigation in Complex  Environments from Temporal Logic Specifications

\thanks{Manuscript received: September, 9, 2022; December, 24, 2022; Accepted February, 1, 2023. (Corresponding author: Mingyu Cai)}%Use only for final RAL version
\thanks{This paper was recommended for publication by Editor Lucia Pallottino upon evaluation of the Associate Editor and Reviewers' comments.}
\thanks{
This work was parially supported by the Under Secretary of Defense for Research and Engineering under Air Force Contract No. FA8702-15-D-0001 at Lehigh University and by the NSF under grant IIS-2024606 at Boston University.}

\thanks{$^{1}$Mingyu Cai and Cristian-Ioan Vasile are with Mechanical Engineering, Lehigh University, Bethlehem, PA, 18015 USA.
        {\tt\footnotesize mic221@lehigh.edu,  
crv519@lehigh.edu}}
\thanks{$^{2}$Erfan Asai and Calin Belta are with Mechanical Engineering Department, Boston University, Boston, MA 02215, USA.
        {\tt\footnotesize eaasi@bu.edu,  
cbelta@bu.edu}}
}

\author{Mingyu Cai$^{1}$, Erfan Aasi$^{2}$, Calin Belta$^{2}$, Cristian-Ioan Vasile$^{1}$

\thanks{Digital Object Identifier (DOI): see top of this page.}}

\markboth{IEEE Robotics and Automation Letters. Preprint Version. February, 2023}
{Cai \MakeLowercase{\textit{et al.}}: DRL for Continuous Control in Cluttered Environments from Temporal Logic Specifications} 

\maketitle

\begin{abstract}
Model-free continuous control for robot navigation tasks using Deep Reinforcement Learning (DRL) that relies on noisy policies for exploration is sensitive to the density of rewards. 
In practice, robots are usually deployed in cluttered environments, containing many obstacles and narrow passageways. Designing dense effective rewards is challenging, resulting in exploration issues during training. 
Such a problem becomes even more serious when tasks are described using temporal logic specifications. This work presents a deep policy gradient algorithm for controlling a robot with unknown dynamics operating in a cluttered environment when the task is specified as a Linear Temporal Logic (LTL) formula. To overcome the environmental challenge of exploration during training, we propose a novel path planning-guided reward scheme by integrating sampling-based methods to effectively complete goal-reaching missions. To facilitate LTL satisfaction, our approach decomposes the LTL mission into sub-goal-reaching tasks that are solved in a distributed manner.
Our framework is shown to significantly improve performance (effectiveness, efficiency) and exploration of robots tasked with complex missions in large-scale cluttered environments. 
A video demonstration can be found on YouTube Channel: \url{https://youtu.be/yMh_NUNWxho}.

\end{abstract}


% \begin{keywords}
% Formal Methods, Robotic Planning, Multi-agent Systems, Optimal Control, Heterogeneous Teams
% \end{keywords}

\section{INTRODUCTION}

%% intro of robot navigation and LTL-guidued tasks
Model-free Deep Reinforcement Learning (DRL) employs neural networks to find optimal policies for unknown dynamic systems via maximizing long-term rewards \cite{mnih2015human}.
In principle, DRL offers a method to learn such policies based on the exploration vs. exploitation trade-off \cite{sutton2018reinforcement}, but the efficiency of the required exploration has prohibited its usage in real-world robotic navigation applications due to natural sparse rewards. To effectively collect non-zero rewards, existing DRL algorithms simply explore the environments, using noisy policies and goal-oriented reward schemes. As the environment becomes cluttered and large-scale, naive exploration
strategies and standard rewards become less effective resulting in local optima. This problem becomes even more challenging for complex and long-horizon robotic tasks. Consequently, the desired DRL-based approaches for robotic target-driven tasks are expected to have the capability of guiding the exploration during training toward task satisfaction. 
% (1) identify states that are worth exploring for the navigation tasks during training; (2) guide the exploration using effective reward functions, toward the task satisfaction. 

\noindent
\textbf{Related Work:}
For exploration in learning processes, many prior works \cite{schulman2015trust,lillicrap2015continuous,schulman2017proximal,lowe2017multi, haarnoja2018soft} employ noise-based exploration strategies integrated with different versions of DRLs, whereas their sampling efficiency relies mainly on the density of specified rewards. 
Recent works \cite{hester2018deep,nair2018overcoming} leverage human demonstrations to address exploration issues for robotic manipulation. On the other hand, the works in \cite{vecerik2017leveraging,fujimoto2019off} focus on effectively utilizing the dataset stored in the reply buffer to speed up the training. However, these advances assume a prior dataset is given beforehand, and they can not be applied to learn from scratch. In cluttered environments containing dense obstacles and narrow passageways, their natural rewards are sparse, resulting in local optimal behaviors and failure to reach destinations.

The common problem is how to generate guidance for robot navigation control in cluttered environments.
Sampling-based planning methods, such as Rapidly-exploring Random Tree (RRT)~\cite{lavalle2001randomized}, RRT*~\cite{karaman2011sampling} and Probabilistic Road Map (PRM)~\cite{lavalle2004relationship}, find collision-free paths over continuous geometric spaces. In robotic navigation, they are typically integrated with path-tracking control approaches such as Model Predictive Control (MPC)~\cite{qin2003survey}, control contraction metrics~\cite{manchester2017control} and learning Lyapunov-barrier functions~\cite{dawson2022safe}.  While these methods are effective, their control designs are model-based. Model-free path tracking, which is the focus of this paper, is still an open problem.

Furthermore, this work also considers complex navigation tasks instead of simple goal-reaching requirements.
 Motivated by task-guided planning and control, formal languages are shown to be efficient tools for expressing a diverse set of high-level specifications \cite{baier2008}. For unknown dynamics, temporal logic-based rewards are developed and integrated with various DRL algorithms. In particular,
deep Q-learning is employed in~\cite{Icarte2018,Camacho2019,hasanbeig2019deepsynth, xu2020joint} 
 over discrete action-space.
 For continuous state-action spaces, 
 the authors in \cite{Li2019,vaezipoor2021ltl2action, icarte2022reward} utilize actor-critic algorithms, e.g., proximal policy optimization (PPO) \cite{schulman2017proximal} for policy optimization, validated in robotic manipulation and safety tasks, respectively. All aforementioned works only study LTL specifications over finite horizons. To facilitate defining LTL tasks over infinite horizons, recent works \cite{Cai2021modular, cai2021safe} improve the results from~\cite{hasanbeig2020deep} by converting LTL into a novel automaton called E-LDGBA, a variant of the Limit Deterministic Generalized B\"uchi Automaton (LDGBA) \cite{Sickert2016}. To improve the performance for the long-term (infinite horizon) satisfaction, the authors propose a modular architecture of Deep Deterministic Policy Gradient (DDPG) \cite{lillicrap2015continuous} to decompose the global missions into sub-tasks. 
However, none of the existing works can address large-scale, cluttered environments, since an LTL-based reward requires the RL-agent
to visit the regions of interest towards the LTL satisfaction. Such sparse rewards can not tackle challenging environments. Sampling-based methods and reachability control synthesis for LTL satisfaction are investigated in~\cite{vasile2020reactive, kantaros2020stylus, srinivasan2020control, luo2021abstraction}.
All assume known system dynamics. In contrast, our paper proposes a model-free approach for LTL-based navigation control in cluttered environments.
% In contrast, from an AI perspective regarding the unknown black box $\mathcal{S}$, we use the geometric RRT* path to synthesize dense rewards, and apply distributed DPGs to learn the optimal policies.

% The work in 
% \cite{liu2021cooperative} proposes a multi-agent exploration method, by restricting and expanding the state space. But it's still unknown whether each agent takes care of its own independent responsibilities.
\vspace{0.1cm}
\noindent
\textbf{Contributions:}   Intuitively, the most effective way of addressing the environmental challenges of learning is to optimize the density of rewards such that the portion of transitions with positive rewards in the reply buffer is dramatically increased. 
To do so, we bridge the gap between sampling-based planning and model-free DPGs to solve standard reachability problems. In particular, we develop a novel exploration guidance technique using geometric RRT* \cite{karaman2011sampling}, to design the rewards. We then propose a distributed DRL framework for LTL satisfaction by decomposing a global complex and long-horizon task into individual reachability sub-tasks.

Moreover, we propose an augmentation method to address the non-Markovianity of the reward design process. Due to unknown dynamics, we overcome the infeasibility of the geometric RRT* guidance. Our algorithm is validated through case studies to demonstrate its performance increase compared to the distance and goal-oriented baselines. We show that our method employing geometric path planning guidance achieves significant training improvements for DRL-based navigation in cluttered environments, where the tasks can be expressed using LTL formulas.
% In summary, the technical contributions of the paper are as follows:
% \begin{compactitem}[$\bullet$]
%     \item Our framework integrates actor-critic RL and RRT* methods and synthesizes model-free optimal policies over a continuous state-action space.
%     % \vspace{0.8mm}
%     \item We propose an automaton method to address the non-Markovian properties of reward design. Due to unknown dynamics, we overcome the infeasibility of the geometric RRT* guidance.
%     % \vspace{0.8mm}
%     \item We validate our algorithm through case studies to demonstrate its better performance compared to the distance and goal-oriented baselines. {\color{blue} We show that our method employing geometric path planning guidance achieves significant training improvements for DRL-based navigation in cluttered environments, where the tasks can be expressed using LTL formulas.}
% \end{compactitem}

% \textcolor{red}{how is complex different from cluttered? what does complex mean?}

\textbf{Organization:} Sec.~\ref{sec:Preliminary} introduces basic concepts of robot dynamics, Markov Decision Processes (MDP) that capture the interactions between robot and environment, RL for solving learning-based control for the MDP model, and LTL for defining robot navigation specifications. In Sec.~\ref{sec:Problem}, we define the problem and emphasize the challenges.
Sec.~\ref{sec:exploration_Solution} presents our approach for addressing simple goal-reaching tasks via reward design using sampling-based methods over the workspace.
In Sec.~\ref{sec:LTL_Solution}, we show how to use and extend the proposed approach to general LTL mission specifications.
The performance of the proposed method is shown in Sec.~\ref{sec:experiment}.

\section{Preliminaries}
\label{sec:Preliminary}

The evolution of a continuous-time dynamic system $\mathcal{S}$ starting from an initial state $s_{0}\in S_{0}\subset S$ is given by
\begin{equation}\label{eqn:robot}
    \dot{s}=f\left(s, a\right),
\end{equation}
where $s\in S\subseteq \mathbb{R}^{n}$ is the state vector in the compact set $S$, $a\in A\subseteq \mathbb{R}^{m}$ is the control input, and $S_0$ is the initial set. The function $f:\mathbb{R}^{n}\times \mathbb{R}^{m}\rightarrow\mathbb{R}^{n}$ is locally Lipschitz continuous and unknown. 

Consider a robot with the unknown dynamics \eqref{eqn:robot}, operating in an environment $Env$ that is represented by a compact subset $X\subset\mathbb{R}^{d}, d\in \left\{2,3\right\}$ as a workspace of the robot.
The relation between $\mathcal{S}$
and $X$ is defined by the projection $Proj:S\shortrightarrow X$. 
The space $X$ contains regions of interest that are labeled by a set of atomic propositions $\AP$. We use $2^{\AP}$ to represent the power set of $\AP$. We denote $L_X:X\shortrightarrow2^{\AP}$ to label regions in the workspace. Let $L:S\shortrightarrow2^{\AP}$ be the induced labeling function over $\mathcal{S}$ and we have $L(s) = L_X(Proj(s))$. 
Note that $S$ represents the robot state space that can be high-dimensional while $X$ is the workspace it is deployed in, i.e., two or three-dimensional Euclidean space.

\textbf{Reinforcement Learning:\space}The interactions between environment $Env$ and the unknown dynamic system $\mathcal{S}$ with the state-space $S$ can be captured by a continuous labeled-MDP (cl-MDP) \cite{thrun2002}. A cl-MDP is a tuple $\mathcal{M}=(S, S_{0}, A,p_{S},\AP,L, R, \gamma)$,
where $S\subseteq\mathbb{R}^{n}$ is the continuous state space, $S_{0}$ is the set of initial states, $A\subseteq\mathbb{R}^{m}$
is the continuous action space, $p_{S}$ represents the unknown system dynamics as a distribution, $AP$ is the set of atomic propositions, $L:S\shortrightarrow2^{\AP}$ is the labeling function, $R:S\times A\times S\shortrightarrow\mathbb{R}$ is the reward function, and $\gamma\in(0,1)$ is the discount factor.
The distribution $p_{S}:\mathfrak{B}\left(\mathbb{R}^{n}\right)\times A\times S\shortrightarrow\left[0,1\right]$
is a Borel-measurable conditional transition kernel, s.t. $p_{S}\left(\left.\cdot\right|s,a\right)$
is the probability measure of the next state given current $s\in S$ and $a\in A$ over the Borel
space $\left(\mathbb{R}^{n},\mathfrak{B}\left(\mathbb{R}^{n}\right)\right)$,
where $\mathfrak{B}\left(\mathbb{R}^{n}\right)$ is the set of all
Borel sets on $\mathbb{R}^{n}$.

Let $\pi(a|s)$ denote a policy that is either deterministic, i.e., $\pi: S \shortrightarrow A$, or stochastic, i.e., $\pi: S\times A\shortrightarrow [0,1]$, which maps states to distributions over actions.
At each episode, the initial state of the robot in $Env$ is denoted by $s_{0}\in S_{0}$. At each time step $t$, the agent  observes the state $s_{t}\in S$ and executes an action $a_{t}\in A$, according to the policy $\pi(a_{t}| s_{t})$, and $Env$ returns the next state $s_{t+1}$ sampled from $p_{S}(s_{t+1}|s_{t}, a_{t})$. The process is repeated until the episode is terminated.
The objective of the robot is to learn an optimal policy $\pi^{*}(a| s)$ that maximizes the expected discounted return $J(\pi)=\mathbb{E}^{\pi}\left[\stackrel[k=0]{\infty}{\sum}\gamma^{k}\cdot R(s_{k},a_{k},s_{k+1})\right]$ under the policy $\pi$.

\textbf{Linear Temporal Logic (LTL):\space} An LTL formula is built 
to describe high-level specifications of a system. Its ingredients are a set of atomic propositions, and combinations of Boolean and temporal operators. The syntax of LTL formulas is defined:
\begin{equation*}
        \phi   :=  \text{true} \, | \, a \, | \, \phi_1 \land \phi_2 \, | \, \lnot \phi_1 | \Next\phi \, | \, \phi_1 \UntilOp \phi_2\:, 
\end{equation*}
where $a\in\AP$ is an atomic proposition, \emph{true}, \emph{negation} $\lnot$, and \emph{conjunction} $\land$ are propositional logic operators, and \emph{next} $\Next$ and \emph{until} $\UntilOp$ are temporal operators. 
Alongside the standard operators introduced above, other propositional logic operators, such as \emph{false}, \emph{disjunction} $\lor$, and \emph{implication} $\rightarrow$, and temporal operators, such as \emph{always} $\Always$ and \emph{eventually} $\Eventually$, are derived from the standard operators.

For a infinite word $o$ starting from the step $0$, let $o_t, t\in\mathbb{N}$ denotes the value at step $t$. The semantics of an LTL formula are interpreted over words, where a word is an
infinite sequence $o=o_{0}o_{1}\ldots$, with $o_{i}\in2^{\AP}$ for
all $i\geq0$.
The satisfaction of an LTL formula $\phi$ by the word $o$ is denote by $o\models\phi$.
More details about LTL formulas can be found in \cite{baier2008}.

In this work, we restrict our attention to LTL formulas that exclude the \emph{next} temporal operator, which is not meaningful for continuous state-action space~\cite{kloetzer2008fully, luo2021abstraction}.

% For a infinite word $o$ starting from the step $0$, let $o(t), t\in\mathbb{N}$ denotes the value at step $t$, and $o[t{:}]$ denotes the word starting from step $t$. 

% The semantics of LTL are defined as~\cite{baier2008}:
% \begin{equation*}
% \arraycolsep=1.4pt
% \begin{array}{lcl}
% % o \models \text{true}  \\
% o \models \pi  & \Leftrightarrow & \pi\in  o(0)  \\
% o \models \phi_{1}\land\phi_{2} &  \Leftrightarrow & o\models \phi_{1} \text{ and } o \models \phi_{2}  \\
% o \models \lnot\phi  & \Leftrightarrow & o \not\models \phi  \\
% o \models \Next\phi  & \Leftrightarrow & o[1{:}] \models\phi  \\
% o \models \phi_1 \UntilOp \phi_2  & \Leftrightarrow & \exists t \text{ s.t. }o[t{:}]\models\phi_{2}, \forall t'\in [0,t),  o[t'{:}]\models\phi_{1}  \\
% \end{array} 
% \end{equation*}

\section{Problem Formulation}
\label{sec:Problem}
% \textcolor{red}{Mingyu: Under writing by Prof. Cristi.}

Consider a cl-MDP $\mathcal{M}=\left(S, S_{0}, A,p_{S},\AP,L, R, \gamma \right)$.
The induced path under a policy $\pi=\pi_{0}\pi_{1}\ldots$ over $\mathcal{M}$ is $\boldsymbol{s}_{\infty}^{\pi}=s_{0}\ldots s_{i}s_{i+1}\ldots$, where $p_{S}(s_{i+1}|s_{i}, a_{i})>0$ if $\pi_{i}(a_{i}| s_{i})>0$.
Let $L\left(\boldsymbol{s}_{\infty}^{\pi}\right)=o_{0}o_{1}\ldots$
be the sequence of labels associated with $\boldsymbol{s}_{\infty}^{\pi}$,
such that $o_{i}= L(s_{i}), \forall i\in \left\{0, 1,2,\ldots\right\}$. We denote the satisfaction relation of the induced %$\boldsymbol{s}_{\infty}^{\pi}$ 
trace with $\phi$ by $L(\boldsymbol{s}_{\infty}^{\pi})\models\phi$. The probability of satisfying $\phi$ under the policy $\pi$, starting from
an initial state $s_{0}\in S_{0}$, is defined as
\[
{\Pr{}_{M}^{\pi}(\phi)=\Pr{}_{M}^{\pi}(L(\boldsymbol{s}_{\infty}^{\pi})\models\phi\,\big|\,\boldsymbol{s}_{\infty}^{\pi}\in\boldsymbol{S}_{\infty}^{\pi}),}
\]
where $\boldsymbol{S}_{\infty}^{\pi}$ is the set of admissible
paths from the initial state $s_{0}$, under the policy ${\pi}$, and the detailed computation of $\Pr{}_{M}^{\pi}(\phi)$ can be found in \cite{baier2008}.
The transition distributions $p_{S}$ of $\mathcal{M}$ are unknown due to the unknown dynamic $\mathcal{S}$, and DRL algorithms are employed to learn the optimal control policies.

In this paper, the cl-MDP $\mathcal{M}$ captures the interactions between a cluttered environment $Env$ with geometric space $X$, and an unknown dynamic system $\mathcal{S}$. Note that explicitly constructing a cl-MDP $\mathcal{M}$ is impossible, due to the continuous state-action space. We track any cl-MDP $\mathcal{M}$ on-the-fly (abstraction-free) using deep neural network, according to the evolution of the dynamic system $\mathcal{S}$ operating in $Env$.

% \begin{problem}
% \label{problem1}
% Consider a goal region and obstacles in $Env$ labeled as $\mathcal{G}$ and $\mathcal{O}$, respectively. The standard obstacle-free goal-reaching mission is expressed as the LTL formula $\phi=\Always\lnot\mathcal{O}\land\Eventually\mathcal{G}$.
% We aim at finding the optimal policy $\pi^*$ of $\mathcal{M}$ that satisfies the specification $\phi$  i.e. $\Pr{}_{M}^{\pi}(\phi)>0$.
% \end{problem}

\begin{problem}
\label{problem2}
Consider a set of labeled goal regions in $Env$ i.e., $\AP_{\mathcal{G}}=\left\{\mathcal{G}_1, \mathcal{G}_2,\ldots\right\}$. The safety-critical specification is expressed as $\phi=\Always\lnot\mathcal{O}\land\phi_{g}$ , where $\mathcal{O}$ denotes the atomic proposition for obstacles.The expression $\phi=\Always\lnot\mathcal{O}\land\phi_{g}$ requires the robot satisfying a general navigation task $\phi_{g}$, e.g., goal-reaching, while avoiding obstacles.
The objective is to synthesize the optimal policy $\pi^*$ of $\mathcal{M}$ satisfying the task $\phi$, i.e., $\Pr{}_{M}^{\pi}(\phi)>0$. 
\end{problem}

\textbf{Assumption 1. } Let $X_{free}$ denote the obstacle-free space. We assume that there exists at least one policy that drives the robot from initial states toward the regions of interest while always operating in $X_{free}$. This is reasonable since the assumption ensures the existence of policies satisfying a given valid LTL specification.

For Problem~\ref{problem2}, typical learning-based algorithms for target-driven problems only assign positive rewards when the robot reaches any goal region $X_{\mathcal{G}_{i}}$ toward the LTL satisfaction, resulting exploration issues of DRL rendered from the environmental challenge. This point is obvious even when considering the special case of goal-reaching tasks, i.e.,$\phi_{p}=\Always\lnot\mathcal{O}\land\Eventually\mathcal{G}_{i}$, where the sub-task $\Eventually\mathcal{G}_{i}$ requires to eventually visit the goal region $X_{\mathcal{G}_{i}}$ labeled as $\mathcal{G}_{i}$.
% Following on Problem~\ref{problem1}, we account another challenge derived from complex tasks.

\begin{example}
\label{example1}
Consider an autonomous vehicle as an RL-agent with unknown dynamics that is tasked with specification $\phi$, shown in Fig.~\ref{fig:running_example} (a). For a goal-reaching task as a special LTL formula $\phi_{p}=\Always\lnot\mathcal{O}\land\Eventually\mathcal{G}_{1}$, if the RL-agent only receives a reward after reaching the goal region $\mathcal{G}_1$, it will be hard to effectively explore using data with positive rewards task and noisy policies. 
The problem becomes more challenging for specifications such as 
$\phi_{ex}=\Always\lnot\mathcal{O}\land\phi_{g, ex}=\Always\lnot\mathcal{O}\land\Always(\left(\Eventually\mathtt{\mathcal{G}_{1}}\land\Eventually\mathtt{\left(\mathcal{G}_{2}\land\Eventually\mathtt{\ldots\land \Eventually\mathcal{G}_{4}}\right)}\right)$
that require the robot to visit regions $\mathcal{G}_1, \mathcal{G}_2, \mathcal{G}_3, \mathcal{G}_4$ sequentially infinitely many times.
\end{example}

In Sec.~\ref{sec:exploration_Solution}, we show how to learn the control policy of completing a standard goal-reaching mission $\phi_{p}=\Always\lnot\mathcal{O}\land\Eventually\mathcal{G}_{i}$ in cluttered environments. Then,  Sec.~\ref{sec:LTL_Solution} builds upon and extends the approach
% extends the novelty and shows how
to solve a general safety-critical navigation task $\phi=\Always\lnot\mathcal{O}\land\phi_{g}$ in a distributed manner.

\begin{figure}[!t]
\begin{center}
\centerline{\includegraphics[width=\columnwidth]{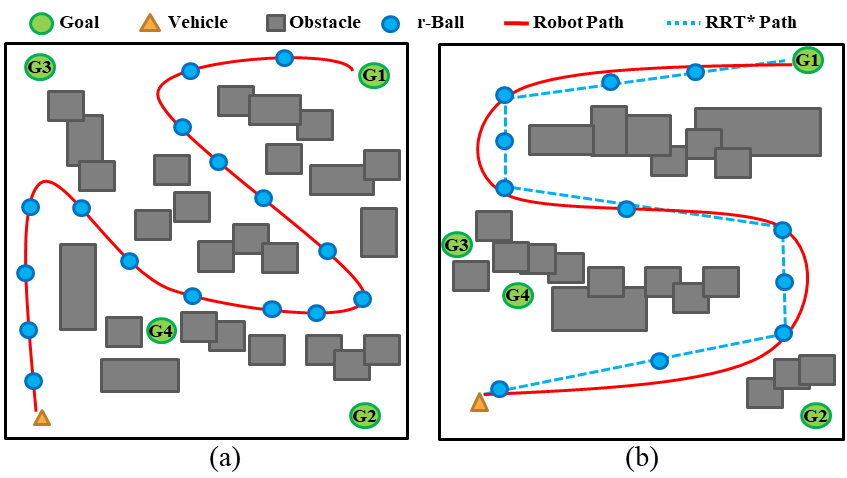}}
\caption{Consider an autonomous vehicle operating in a cluttered and large-scale environment. It is challenging to learn optimal policies as described in Example~\ref{example1}. The trajectories provide insights of reward design in Sec.~\ref{sec:exploration_Solution}.}
\label{fig:running_example}
\end{center}
 \vspace{-0.5cm}
\end{figure}

\section{Overcoming Exploration\label{sec:exploration_Solution}
}

% To overcome the exploration challenges of DRL in complex environments, this section proposes a solution for the standard goal-reaching task $\phi_{p}=\Always\lnot\mathcal{O}\land\Eventually\mathcal{G}_{i}$ by designing a sampling-based reward scheme. 

In Sec.~\ref{subsec: RRT*}, we briefly introduce the geometric sampling-based algorithm to generate an optimal path for the standard goal-reaching tasks $\phi_{p}=\Always\lnot\mathcal{O}\land\Eventually\mathcal{G}_{i}$.
Subsequently, Sec.~\ref{subsec: reward} develops a novel dense reward to overcome the exploration challenges and provides rigorous analysis for learning performance.

\subsection{Geometric RRT*\label{subsec: RRT*}}

The standard optimal RRT* method \cite{karaman2011sampling} is a variant of RRT \cite{lavalle2001randomized}.
Both RRT and RRT* are able to handle the path planning in cluttered and large-scale environments. 
Due to its optimality, we choose RRT* over RRT, to improve the learned performance of the optimal policies. 
Since the dynamic system $\mathcal{S}$ is unknown, we use the geometric RRT* method that builds a tree $G=(V,E)$ incrementally in $X$, where $V$ is the set of vertices and $E$ is the set of edges. If $V$ intersects the goal region, we find a geometric trajectory to complete the task $\phi$.
% otherwise, the planning problem is infeasible. 
The detailed procedure of geometric RRT* is described in the Appendix~\ref{append:RRT*}. 
Here, we briefly introduce two of the functions in the geometric RRT* method, which are used in explaining our method in the next sections.

{\em Distance and cost:} The function $dist: X\times X\rightarrow [0,\infty)$ is the metric that computes the geometric Euclidean distance between two states. The function $Cost: X\rightarrow [0,\infty)$ returns the length of the path in the tree between two input states.

{\em Steering:} Given two states $x$ and $x'$, the function $Steer$ returns a state $x_{new}$ such that $x_{new}$ lies on the geometric line connecting $x$ to $x'$, and its distance from $x$ is at most $\eta$, i.e., $dist(x, x')\leq\eta$, where $\eta$ is an user-specified parameter. In addition, the state $x_{new}$ must satisfy $dist(x_{new}, x')\leq dist(x, x')$.

Having the tree $G=(V,E)$ generated by the RRT* method, if there exists at least one node $x\in V$ that is located within the goal regions, we find the optimal state trajectory satisfying the task $\phi$ as a sequence of geometric states $\boldsymbol{x}^*=x_{0}x_{1}\ldots x_{N_{p}}$, where $x_{N_{p}}\in X_{\mathcal{G}}$ and $x_{i}\in V$, $\forall i=0,1\ldots,N_{p}$. 

\subsection{Sampling-based reward\label{subsec: reward}}

Here we use the optimal geometric trajectory $\boldsymbol{x}^*$ and the properties of the generated tree $G=(V, E)$ to synthesize the reward scheme.
First, let $x|\boldsymbol{x}^*$ denote a state of $\boldsymbol{x}^*$, and the total length of $\boldsymbol{x}^*$ in the tree $G$ be equal to $Cost(x_{N_{p}}|\boldsymbol{x}^*)$. We define the distance from each state $x\in\boldsymbol{x}^*$ to the destination $x_{N_{p}}$ as $Dist(x|\boldsymbol{x}^*) = Cost(x_{N_{p}}|\boldsymbol{x}^*)-Cost(x|\boldsymbol{x}^*)$.
Based on the distance, we design the RRT* reward scheme to guide the learning progress towards the satisfaction of $\phi$.
Reaching an exact state in the continuous state space is challenging for robots.
Thus, we define the norm $r$-ball for each state $x|\boldsymbol{x}^*$ to allow the robot to visit the neighboring area of the state as $Ball_{r}(x|\boldsymbol{x}^*)=\left\{x'\in X\mid dist(x|\boldsymbol{x}^*,x')\leq r\right\}$, where $x|\boldsymbol{x}^*$ is the center and $r$ is the radius. For simplicity, we select $r\leq\frac{\eta}{2}$ based on the steering function of the geometric RRT*, such that the adjacent r-balls along the optimal trajectory $\boldsymbol{x}^*$ do not intersect with each other.

We develop a progression function $D:X\rightarrow [0,\infty)$ to track whether the current state is getting closer to the goal region, by following the sequence of balls along $\boldsymbol{x}^*$ as:
\begin{equation}
D(x)=
\begin{cases}
Dist(x_{i}|\boldsymbol{x}^*) & \text{if } x\in Ball_{r}(x_{i}|\boldsymbol{x}^*)\\
\infty & \text{otherwise}
\end{cases}
\label{eq:progression}
\end{equation}
For the cl-MDP $\mathcal{M}$ capturing the interactions between $\mathcal{S}$ and $Env$, the intuition behind the sampling-based reward design is to assign a positive reward whenever the robot gets geometrically closer to the goal region, along the optimal path obtained by the RRT* (Alg.~\ref{alg:RRT*}).

During each episode of learning, a state-action sequence $s_{0}a_{0}s_{1}a_{i}\ldots s_{t}$ up to current time $t$ is projected into the state and action sequences  $\boldsymbol{s}_{t}=s_{0}s_{1}\ldots s_{t}$ and $\boldsymbol{a}_{t}=a_{0}a_{1}\ldots a_{t-1}$, respectively.
We define
\vspace{-4pt}
\begin{equation*}
D_{min}(\boldsymbol{s}_{t})=\underset{s\in \boldsymbol{s}_{t}}{\min}\left\{D(Proj(s))\right\} 
\end{equation*}
\vspace{-4pt}
as the progression energy that is equal to the minimum distance along the optimal path towards the destination, up to step $t$. The objective of the reward is to drive the robot such that $D_{min}(\boldsymbol{s}_{t})$ decreases. 
However, employing the function $D_{min}(\boldsymbol{s}_{t})$ for reward design that depends on the history of the trajectory results in a non-Markovian reward function~\cite{Icarte2018}, while the policy $\pi(s)$ only takes the current state as input and can not distinguish the progress achieved by the histories $\boldsymbol{s}_{t}$.

To address the issue, inspired by the product MPD~\cite{baier2008}, given the history $\boldsymbol{s}_{t}$, we keep tracking the index $i_{t}\in\left\{0,1\ldots,N_{p}\right\}$ of the center state of the visited $r$-ball regions $Ball_{r}(x_{i_{t}}|\boldsymbol{x}^*)$ with minimum distance $Dist(x_{i_{t}}|\boldsymbol{x}^*)$ deterministically, i.e.,
% \begin{equation*}
% x_{i_{t}}=\left\{Proj(s_{i_{t}})\Bigl| s_{i_{t}}=\underset{s\in \boldsymbol{s}_{t}}{\arg\min}\left\{D(Proj(s))\right\}\right\}      
% \end{equation*}
\begin{equation*}
x_{i_{t}}=Proj(s_{i_{t}}), \text{ where } s_{i_{t}}=\underset{s\in \boldsymbol{s}_{t}}{\arg\min}\left\{D(Proj(s))\right\}
\end{equation*}

If none of the $r$-balls are visited up to $t$, we set $i_{t}=0$. Then, the current state $s_{t}$ is embedded with the index $i_{t}$ as a product state $s^{\times}_{t}=(s_{t}, i_{t})$, which is considered as the input of the policy, i.e., $\pi(s^{\times}_{t})$. Note that we treat the embedded component $i_{t}$ as the state of a deterministic automaton~\cite{baier2008}. The relevant analysis can be found in Appendix~\ref{append:product}.

Let $R: s^{\times}\rightarrow \mathbb{R}$ denote the episodic reward function. We propose a novel scheme to assign the Markovian reward with respect to the product state $s^{\times}_{t}$ as: 
\begin{equation}
R(s^{\times}_{t})=\left\{ \begin{array}{cc}
R_{-}, & \text{if } Proj(s_{t})\in X_{\mathcal{O}},\\
R_{++}, & \text{if } D(Proj(s_{t}))=0,
\\
R_{+}, & \text{if } D(Proj(s_{t}))< D_{min}(\boldsymbol{s}_{t-1}),\\
0, & \text{otherwise,}
\end{array}\right.\label{eq:reward_function}
\end{equation}
where $R_{+}$ is a positive constant reward, $R_{++}$ is a boosted positive constant that is awarded when the robot reaches the destination, and $R_{-}$ is the negative constant reward that is assigned when the robot violates the safety task of $\phi$, i.e., $\phi_{safe}=\Always\lnot\mathcal{O}$. Note that if the robot crosses both obstacles and $r$-balls, it receives the negative reward $R_{-}$, which has first priority. This setting does not restrict selections of the parameter $r$ (radius of $r$-balls) for implementations.

\begin{example}
\label{example2}
As shown in Fig.~\ref{fig:running_example} (a), we apply the RRT* method to generate the optimal trajectory (shown in red) in a challenging environment. Then, we construct the sequence of $r$-balls along it and apply the reward design~\eqref{eq:reward_function} to guide the learning and overcome exploration difficulties.
\end{example}

\begin{rem}
Since geometric RRT* does not consider the dynamic system, the optimal path $\boldsymbol{x}^*$ may be infeasible for the robot to follow exactly, with respect to any policy. As a running example in Fig.~\ref{fig:running_example} (b), our RRT* reward is robust such that the robot is not required to strictly follow all $r$-balls of the optimal path.
Instead, in order to receive the positive reward, the robot only needs to move towards the destination and pass through the partial $r$-balls $Ball_{r}(x_{i}|\boldsymbol{x}^*), i\in\left\{0,1\ldots,N_{p}\right\}$ along the optimal path. If we want the robot to reach the goal with desired orientations, we can add another reward that measures the errors between the robot's actual and desired orientations.
\end{rem}

By applying the reward design~\eqref{eq:reward_function}, we formally verify the performance of the reward~\eqref{eq:reward_function} for the reach-avoid task $\phi_{p}$.

\begin{thm}
\label{thm:RRT*}
If Assumption 1 holds, by selecting $R_{++}$ to be sufficiently larger than $R_{+}$, i.e., $R_{++} \gg R_{+}$, any algorithm that optimizes the expected return $J(\pi)$ is guaranteed to find the optimal policy $\pi^{*}$ satisfying the goal-reaching task $\phi_{p}$,
i.e., $\Pr{}_{M}^{\pi^{*}}(\phi_{p})>0$.
\end{thm}

The proof is presented in Appendix~\ref{proof:thm1}.
Theorem~\ref{thm:RRT*} provides a theoretical guarantee for the optimization performance, allowing us to apply practical algorithms to find the approximated optimal policy in continuous space. 

Based on Theorem~\ref{thm:RRT*} and regarding the continuous control task, we apply advanced DRL methods, e.g., actor-critic algorithms~\cite{lillicrap2015continuous, schulman2017proximal, haarnoja2018soft}, to find the optimal policy $\pi^*$.
Consider a policy $\pi_{\theta}(a|s^{\times})$, parameterized by $\theta$, the learning objective aims to find the optimal policy via optimizing the parameters $\theta$ and maximizing the expected discount return $J(\theta)=\mathbb{E}^{\pi_{\theta}}\left[\stackrel[t=0]{\infty}{\sum}\gamma^{t}\cdot R(s^{\times}_{t})\right]$, which minimizes the loss function:
\begin{equation}
    \mathcal{L}(\theta)=\mathbb{E}_{(s^{\times}, a, r, s'^{\times})\backsim\mathcal{D}}[(Q(s,a|\omega)-y)^{2}]\label{eq:loss},
\end{equation}
% {\color{blue} [Cristi] the reward $r$ below clashes with the notation for the radius in RRT$^*$. Consider using capital letters $R$, $R_+$, $R_-$, $R_{++}$ for reward. At least for the collected reward in the experience buffer.}
where $\mathcal{D}$ is the reply buffer that stores experience tuples $(s^{\times}, a, R, s'^{\times})$, $Q(s,a|\omega)$ is the state-action valuation function parameterized by $\omega$, and $y=r+\gamma\, Q(s'^{\times},a'|\omega)$. As observed in~\eqref{eq:loss}, actor-critics rely on effective data in the replay buffer, or sample efficiency of the state distribution to minimize the loss function. Due to its high reward density over geometric space, the sampling-based reward is easy to explore and improves the training performance using noisy policies.

% To learn policies over continuous state-action space, this work focuses on DPG-based RL methods, which estimates the optimal expected return $J(\theta)$ with respect to the policy based on sampled trajectories.

% For the stochastic policy, the gradient is estimated as follows:
% \begin{equation}
%     \nabla_{\theta}J(\theta)=\mathbb{E}_{s^{\times}\backsim p^{\pi},a \backsim\pi_{\theta}}\left[\nabla_{\theta}\log\pi_{\theta}(a|s^{\times})Q^{\pi}(\tau)\right]\label{eq:stochastic},
% \end{equation}
% where $p^{\pi}$ is the state distribution.
% It is possible to extend the DPG framework to deterministic policies \cite{silver2014deterministic}. In particular, we formulate the gradient as 
% \begin{equation}
%     \nabla_{\theta}J(\theta)=\mathbb{E}_{s^{\times} \backsim \mathcal{D}}\left[\nabla_{\theta}\pi_{\theta}(a|s^{\times})\nabla_{a}Q^{\pi}(s^{\times},a)|_{a=\pi_{\theta}(s^{\times})}\right].\label{eq:deterministic}
% \end{equation}
% The DPG theorem gives rise to several practical algorithms, leading to a variety of actor-critic algorithms \cite{sutton2018reinforcement}.
% e.g., DDPG\cite{lillicrap2015continuous}, PPO~\cite{schulman2017proximal}, SAC \cite{haarnoja2018soft} etc. 
% As observed in~\eqref{eq:loss}, all DPGs, e.g., \eqref{eq:stochastic} and~\eqref{eq:deterministic}, rely on effective data in the replay buffer, or sample efficiency of the state distribution to minimize the loss function. 
% In this paper, the sampling-based reward improves the exploration performance with any actor-critic algorithm using noisy policies, due to its high reward density.

\begin{thm}
\label{thm:DPG_RRT*}
If Assumption 1 holds,  by selecting $R_{++}$ to be sufficiently larger than $R_{+}$, i.e., $R_{++} \gg R_{+}$, a suitable DPG algorithm that optimizes the expected return $J(\theta)$, finds the optimal parameterized policy $\pi^{*}_{\theta}$ satisfying the LTL tasks $\phi_{p}$, i.e., $\Pr{}_{M}^{\pi^{*}}(\phi_{p})>0$ in the limit.
\end{thm}
Theorem~\ref{thm:DPG_RRT*} is an immediate result of Theorem~\ref{thm:RRT*} and the nature of nonlinear regressions in deep neural networks. In practice, the number of episodes and steps are limited and training has to be stopped eventually.

\section{LTL Task Satisfaction}
\label{sec:LTL_Solution}
 Sec.~\ref{subsec:TL-RRT*} describes how to generate and decompose the optimal path of satisfying general LTL task $\phi$ in a sequence of paths of completing goal-reaching missions $\phi_{p}$, and Sec.~\ref{subsec:DPGs} explains how to integrate distributed DPGs with the novel exploration guidance of Sec.~\ref{sec:exploration_Solution}, to learn the optimal policy.

% {\color{blue} [Cristi: It seems that you use geometric space and workspace interchangeably.  They are not the same. Reviewers will remark on that, especially if they have background in motion planning.]}

\subsection{Geometric TL-RRT*\label{subsec:TL-RRT*}}

Due to the unknown dynamic system, we define the transition system over the geometric space $X$, referred as Geometric-Weighted Transition System (G-WTS).

\begin{defn}
\label{def:WTS}
A G-WTS of $Env$ is a tuple $\mathcal{T}=(X, x_{0}, \rightarrow_{\mathcal{T}}, \AP, L_{X}, C_{\mathcal{T}})$, where $X$ is the geometric space of $Env$, $x_{0}$ is the initial state of robot; $\rightarrow_{\mathcal{T}}\subseteq X\times X$ is the geometric transition relation s.t. $x\rightarrow_{\mathcal{T}}x'$ if $dist(x, x')\leq\eta$ and the straight line $\sigma$ connecting $x$ to $x_{new}$ is collision-free, $\AP$ is the set of atomic propositions as the labels of regions, $L_X: X\rightarrow\AP$ is the labeling function that returns an atomic proposition satisfied at a location $x$, and $C_{\mathcal{T}}: (\rightarrow_{\mathcal{T}})\rightarrow\mathbb{R}^{+}$ is the geometric Euclidean distance, i.e., $C_{\mathcal{T}}(x,x')=dist(x,x'),\forall (x,x')\in\rightarrow_{\mathcal{T}}$.
\end{defn}

The standard WTS~\cite{kloetzer2008fully,luo2021abstraction} defines the transition relations  $x\rightarrow_{\mathcal{T}}x'$ according to the existence of model-based controllers that drive the robot between neighbor regions $x, x'$. Differently, we only consider the geometric connection among states in a model-free manner.

Let $\tau_{\mathcal{T}}=x_{0}x_{1}x_{3}\dots$ denote a valid run of $\mathcal{T}$. An LTL formula $\phi$ can be converted into a Non-deterministic B\"uchi Automata (NBA) to verify its satisfaction.

\begin{defn}
\label{def:NBA}
\cite{vardi1986automata} An NBA over $2^{\AP}$ is a tuple $\mathcal{B}=(Q, Q_{0}, \Sigma, \rightarrow_{\mathcal{B}}, Q_{F})$, where $Q$ is the set of states, $Q_{0}\subseteq Q$ is the set of initial states, $\Sigma=2^{\AP}$ is the finite alphabet, $\rightarrow_{\mathcal{B}}\subseteq Q\times\Sigma\times Q$ is the transition relation, and $Q_{F}\subseteq Q$ is the set of accepting states.
\end{defn}
A valid infinite run $\tau_{\mathcal{B}}=q_{0}q_{1}q_{2}\ldots$ of $\mathcal{B}$ is called accepting, if it intersects with $Q_{F}$ infinite often. Infinite words $\tau_{o}=o_{0}o_{1}o_{2}\ldots, \forall o\in 2^{\AP}$ generated from an accepting run satisfy the corresponding LTL formula $\phi$. An LTL formula is converted into NBA using the tool~\cite{gastin2001fast}. As in~\cite{kantaros2020stylus}, we prune the infeasible transitions of the resulting NBA to obtain the truncated NBA.

\begin{defn}
\label{def:PBA}
Given the G-WTS $\mathcal{T}$ and the NBA $\mathcal{B}$, the product B\"uchi automaton (PBA) is a tuple $P=\mathcal{T}\times\mathcal{B}=(Q_P, Q^{0}_{P}, \rightarrow_{P}, Q^{F}_{P}, C_{P}, L_{P})$, where $Q_{P}=X\times Q$ is the set of infinite product states, $Q^{0}_{P}=x_{0}\times Q_{0}$ is the set of initial states; $\rightarrow_{P}\subseteq Q_{P}\times 2^{\AP}\times Q_{P}$ is the transition relation defined by the rule: $\frac{x\rightarrow_{\mathcal{T}}x'\land\text{ } q\overset{L_X(x)}{\rightarrow_{\mathcal{B}}}q'}{q_{P}=(x,q)\rightarrow_{P} q_{P}'=(x',q')}$, where $q_{P}\rightarrow_{P} q_{p}'$ denotes the transition $(q_{P},q_{P}')\in\rightarrow_{P}$, $Q^{F}_{P}=X\times Q_{F}$ is the set of accepting states, $C_{P}\colon (\rightarrow_{P})\rightarrow\mathbb{R}^{+}$ is the cost function defined as the cost in the geometric space, e.g., $C_{P}(q_{p}=(x,q),q_{p}'=(x',q'))=C_{\mathcal{T}}(x,x'), \forall (q_{P},q_{P}')\in\rightarrow_{P}$, and $L_{P}\colon Q_{P}\rightarrow\AP$ is the labelling function s.t. $L_P(q_{P})=L_X(x), \forall q_{P}=(x,q)$.
\end{defn}

A valid trace $\tau_{P}=q^{0}_{P}q^{1}_{P}q^{2}_{P}\ldots$ of a PBA is called accepting, if it visits $Q^{F}_{P}$ infinitely often, referred as the acceptance condition. Its accepting words $\tau_{o}=o_{0}o_{1}o_{2}\ldots, \forall o_{i}=L_{P}(q^{i}_{P})$ satisfy the corresponding LTL formula $\phi$. Let $\tau_{F}$ denote an accepting trace, and $proj|_{X}: Q_{P}\rightarrow X$ is a function that projects product state space into the workspace, i.e., $proj|_{X}(q_{p})=x, \forall q_{P}=(x,q)$. Using the projection, we extract the geometric trajectory $\tau_{\mathcal{T}}=proj|_{X}(\tau_{F})$ that satisfies the LTL formula.
More details are presented in~\cite{baier2008}. Therefore, the planning objective is to find an acceptable path $\tau_{P}$ of PBA, with minimum cumulative geometric cost $C_{P}$.

However, the state space of G-WTS and PBA are both infinite. Consequently, we are not able to apply a graph search method to a PBA with infinite states.
Thanks to the TL-RRT* algorithm~\cite{luo2021abstraction} for providing an abstraction-free method, it allows us to incrementally build trees that explore the product state-space and find the feasible optimal accepting path. The procedure applies the sampling-based method over the PBA, and is inspired by the fact that the accepting run $\tau_{F}$ admits a lasso-type sequence in the form of prefix-suffix structure, i.e., $\tau_{F}=\tau^{pre}_{P}[\tau^{suf}_{P}]^{\omega}$, where the prefix part $\tau^{pre}_{P}=q^{0}_{P}q^{1}_{P}\ldots q^{K}_{P}$ is only executed once, and the suffix part $\tau^{suf}_{P}=q^{K}_{P}q^{K+1}_{P}\ldots q^{K+M}_{P}$ with $q^{K}_{P}=q^{K+M}_{P}$ is executed infinitely often.

Following this idea, we build the trees for the prefix and suffix paths, respectively. To satisfy the acceptance condition, the set of goal states of the prefix tree $G^{pre}_{P}=(V^{pre}_{P},E^{pre}_{P})$ is defined as $Q^{pre}_{goal} = \left\{q_{P}=(x,q)\in X_{free}\times Q\subseteq Q_{P}\mid q\in Q_{F}\right\}$, where $X_{free}$ is the collision-free geometric space. After obtaining the prefix tree, we construct the set $Q^{*}_{goal}=V^{pre}_{P}\cap  Q^{pre}_{goal}$, and compute the optimal prefix path $\tau^{*}_{pre}$ reaching a state $q^{*}_{P}\in Q^{*}_{goal}$ from the root $q^{0}_{P}$. The suffix tree $G^{suf}_{P}=(V^{suf}_{P},E^{suf}_{P})$ is built by treating $q^{*}_{P}=(x^{*}, q^{*})$ as the root, and its goal states are:
\begin{equation*}
\begin{array}{c}
    Q^{suf}_{goal}(q^{*}_{P})=\left\{\right.q_{P}=(x,q)\in X_{free}\times Q\subseteq Q_{P}\mid \\
    x\rightarrow_{\mathcal{T}}x^{*}\land\text{ } q\overset{L_X(x)}{\rightarrow_{\mathcal{B}}}q^{*}\left\}\right..
\end{array}
\end{equation*}
$Q^{suf}_{goal}(q^{*}_{P})$ collects all states that can reach the state $q^{*}_{P}$ via one transition, and this way it ensures the feasible cyclic path matching the suffix structure. Finally, we search the optimal suffix path $\tau^{*}_{suf}$, by constructing $V^{suf}_{P}\cap Q^{suf}_{goal}$.

\begin{algorithm}[tb]
   \caption{ LTL-RRT*-Distributed DPGs}
   \label{alg:framework}
\begin{algorithmic}[1]
\State \textbf{Input:} $Env$, $\phi=\Always\lnot\mathcal{O}\land\phi_{g}$, Black-box $\mathcal{S}$;
\State \textbf{Initialize:} Geometric space $X$, Primitives of TL-RRT*; 
\State Convert $\phi$ into NBA $\mathcal{B}$
\State Build the incremental trees for PBA geometrically, based on definition~\ref{def:WTS} and definition~\ref{def:PBA}
\State Generate the optimal trajectory $\tau^{*}_{F}=\tau^{*}_{pre}[\tau^{*}_{suf}]^{\omega}$
\State Reformulate the trajectory into the modular form
\vspace{-6pt}
\begin{equation*}
    \mathcal{R}_{F}=(\mathcal{R}_{0}\mathcal{R}_{1}\ldots \mathcal{R}_{K})(\mathcal{R}_{K+1}\ldots \mathcal{R}_{K+l})^{\omega}\label{eq:lasso}
\end{equation*}
\vspace{-20pt}
\For{$i=1,\ldots,K+l$}
\State Construct the RRT* reward based     on~\eqref{eq:reward_function} for $\mathcal{R}_{i}$
\State Assign an actor-critic DPG e.g., DDPG, PPO, for $\mathcal{R}_{i}$
\EndFor

\State Assign the rewards~\eqref{eq:reward_function} and DPGs for each $\mathcal{R}_{i}$.
\State  Train the distributed DPGs in parallel
\State Extract the optimal policy $\pi^{*}_{i}$ from each DPG $\mathcal{R}_{i}$
\State Concatenate all optimal policies in the form
\vspace{-6pt}
\begin{equation*}
\pi^{*}_{\theta}=(\pi^{*}_{0}\pi^{*}_{1}\ldots \pi^{*}_{K})(\pi^{*}_{K+1}\ldots \pi^{*}_{K+l})^{\omega}\label{eq:lasso_policy}
\end{equation*}

\end{algorithmic}
\vspace{-0.2cm}
\end{algorithm}

\subsection{Distributed DPGs\label{subsec:DPGs}}

In this section, we first employ the optimal geometric path $\tau^{*}_{F}=\tau^{*}_{pre}[\tau^{*}_{suf}]^{\omega}$ from Sec.~\ref{subsec: RRT*}, to propose a distributed reward scheme. Since the policy gradient strategy suffers from the variance issue and only finds the optimal policy in the limit (see Theorem~\ref{thm:DPG_RRT*}), instead of directly applying the reward design~\eqref{eq:reward_function} for the whole path $\tau^{*}_{F}$, we decompose it into sub-tasks.
To do so, we divide $\tau^{*}_{F}$ into separated consecutive segments, each of which shares the same automaton components, i.e., $\tau^{*}_{F}=\tau^{*}_{0}\tau^{*}_{1}\ldots\tau^{*}_{K}[\tau^{*}_{K+1}\ldots \tau^{*}_{K+l}]^{\omega}$ such that all states of each sub trajectory $\tau^*_i$ have the same automaton components.
Each segment can be treated as a collision-free goal-reaching problem, denoted as $\mathcal{R}_{i}(\mathcal{G}_i,\mathcal{O})$, where $\mathcal{G}_i$ is label of the $i^{th}$ goal region.
Specifically, suppose the state trajectory of each $\mathcal{R}_{i}(\mathcal{G}_i, \mathcal{O})$ is $\tau^{*}_{i}=q^{0}_{P,i}q^{1}_{P,i}\ldots q^{N_{i}}_{P,i}$, we select the region labeled as $L_{P}(q^{N_{i}}_{P,i})$ containing the
geometric state $proj|_{X}(q^{N_{i}}_{P,i})$.

We show an example of the optimal decomposition in Fig.~\ref{fig:decomposition_example}, where the LTL task $\phi_{1,inf}=\Always\lnot\mathcal{O}\land\phi_{g_{1}}$ over infinite horizons with $\phi_{g_{1}}=\Always\Eventually\mathcal{G}_{1}\land\Always\Eventually\mathcal{G}_{2}\land\Always\Eventually\mathcal{G}_{3}$ that requires to infinitely visit goal regions labeled as $\mathcal{G}_{1}, \mathcal{G}_{2}, \mathcal{G}_{3}$. 
The resulting truncated NBA and decomposed trajectories of TL-RRT* are shown in Fig.~\ref{fig:decomposition_example} (a) and (b), respectively, where decomposed sub-paths are expressed as $\mathcal{R}_{F}=\mathcal{R}_{red}(\mathcal{R}_{blue}\mathcal{R}_{pink} \mathcal{R}_{brown})^{\omega}$, such that the distributed DPGs are applied to train the optimal sub-policies for each one in parallel.

The lasso-type optimal trajectory is reformulated as: $   \mathcal{R}_{F}=(\mathcal{R}_{0}\mathcal{R}_{1}\ldots \mathcal{R}_{K})(\mathcal{R}_{K+1}\ldots \mathcal{R}_{K+l})^{\omega}$.
% \begin{equation}
%     \mathcal{R}_{F}=(\mathcal{R}_{0}\mathcal{R}_{i}\ldots \mathcal{R}_{K})(\mathcal{R}_{K+1}\ldots \mathcal{R}_{K+l})^{\omega}\label{eq:lasso}
% \end{equation}
For the cl-MDP $\mathcal{M}$, we treat each $\mathcal{R}_{i}$ as a task $\phi_{\mathcal{R}_{i}}=\Always\lnot\mathcal{O}\land\Eventually\mathcal{G}_{i}$ solved in Sec.~\ref{subsec: reward}.
In particular, we propose collaborative team of RRT* rewards in~\eqref{eq:reward_function} for each sub-task and assign distributed DPGs for each $\mathcal{R}_{i}$ that are trained in parallel.
After training, we extract the concatenate policy $\pi^{*}_{i}$ of each $\mathcal{R}_{i}$ to obtain the global optimal policy as $\pi^{*}_{\theta}=(\pi^{*}_{0}\pi^{*}_{1}\ldots \pi^{*}_{K})(\pi^{*}_{K+1}\ldots \pi^{*}_{K+l})^{\omega}$. The overall procedure is summarized in Alg.~\ref{alg:framework}, and a detailed diagram with rigorous analysis is presented in Appendix~\ref{append:diagram}. Based on the decomposition properties and Theorem~\ref{thm:DPG_RRT*}, we can conclude that the concatenated optimal policy of Alg.~\ref{alg:framework} satisfies the global LTL specification.

\begin{figure}
\begin{center}
\centerline{\includegraphics[width=\columnwidth]{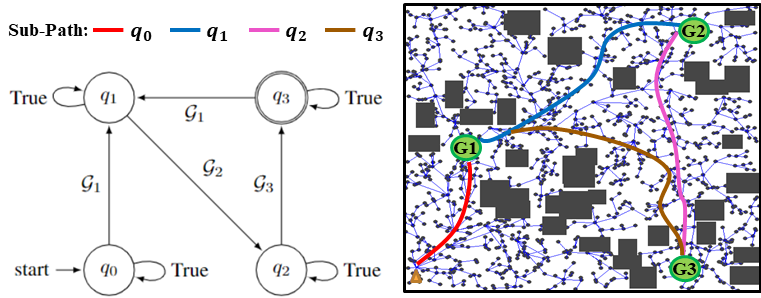}}
\caption{ Decomposition example. (Left) Truncated NBA $\mathcal{B}$ of the LTL formula $\phi_{g_{1}}=\Always\Eventually\mathcal{G}_{1}\land\Always\Eventually\mathcal{G}_{2}\land\Always\Eventually\mathcal{G}_{3}$ for $\phi_{1,inf}=\Always\lnot\mathcal{O}\land\phi_{g_{1}}$; (Right)  Decomposed sub goal-reaching tasks.}
\label{fig:decomposition_example}
\end{center}
\end{figure}

\section{Experimental Results\label{sec:experiment}}

We evaluate the framework on different nonlinear dynamic systems tasked to satisfy various LTL specifications.
The tests focus on large-scale cluttered environments that generalize the simple ones to demonstrate the method's performance.
Obstacles are randomly sampled.
% and we also apply it to simple gym environments~\cite{lowe2017multi} to demonstrate its efficiency.  
We integrate all baselines with either DDPG or SAC as DPG algorithms.
Finally, we show that our algorithm improves the success rates of task satisfaction over both infinite and finite horizons in cluttered environments, and significantly reduces training time for the task over finite horizons.
Detailed descriptions of environments and LTL tasks will be introduced.

Recall that $\Eventually$ denotes the \emph{eventually} operator used to specify feasibility properties (e.g., goal-reaching), while $\Always$ stand for the \emph{always} operator used for safety (e.g., collision avoidance) and invariance (e.g., geo-fencing) properties.

\textbf{Baseline Approaches:}  We refer to our distributed framework as "RRT*" or "D-RRT*", and compare it against three baselines: (i) The TL-based rewards in ~\cite{hasanbeig2020deep,Cai2021modular} referred as "TL",  for the single LTL task,  have shown excellent performance in non-cluttered environments, which generalizes the cases of finite horizons in existing literature~\cite{Li2019,vaezipoor2021ltl2action, icarte2022reward};
(ii) Similar as~\cite{long2018towards, dawson2022learning}, for the goal-reading task $\phi$, the baseline referred to as "NED" designs the reward based on the negative Euclidean distance between the robot and destination; (iii) For a complex LTL task, instead of decomposition, this baseline directly apply the reward scheme~\eqref{eq:reward_function} for the global trajectory $\tau^{*}_{F}=\tau^{*}_{pre}[\tau^{*}_{suf}]^{\omega}$ referred as "G-RRT*".  Note that we focus on comparing the baseline "NED" for finite-horizon tasks and the baseline "G-RRT*" for infinite-horizon and complex tasks.
% The TL-based rewards are widely applied in existing literature~\cite{Li2019, hasanbeig2020deep,Cai2021modular, vaezipoor2021ltl2action} referred as "TL", which is a type of point-to-point rewards, and require the RL-agent to visit the goal regions in the manner of LTL instructions e.g., automaton transitions~\cite{baier2008} or task progression~\cite{bacchus2000using}. Here, we implement the algorithms~\cite{hasanbeig2020deep,Cai2021modular} being able to handle infinite-horizon tasks; 
% For fair comparisons, baseline "NED" only records the boosted positive reward and negative rewards of obstacle collisions.

\begin{figure}
\begin{center}
\centerline{\includegraphics[width=\columnwidth]{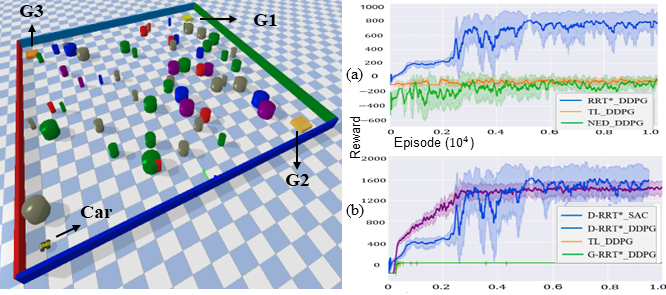}}
\caption{ Baselines comparison of tasks $\phi_{2,fin}$ (a) and $\phi_{2,inf}$ (b) in the Pybullet environment.}
\label{fig:Car_model}
\end{center}
\end{figure}

\begin{figure}
\begin{center}
\centerline{\includegraphics[width=\columnwidth]{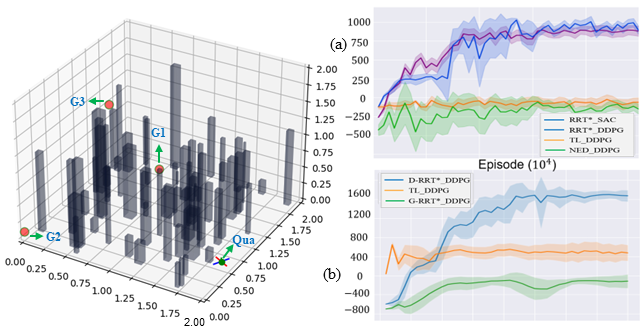}}
\caption{ Baselines comparison of tasks $\phi_{1,fin}$ in (a) and $\phi_{1,inf}$  in (b) of the cluttered 3D quadrotor environment.}
\label{fig:Quadrator}
 \vspace{-1.0cm}
\end{center}
\end{figure}

\textbf{6.1 Autonomous Vehicle\text{ }} We first implement the car-like model of Pybullet\footnote{\url{https://pybullet.org/wordpress/}} physical engine shown in Fig.~\ref{fig:Car_model}. We consider the sequential LTL task and surveillance LTL task over both finite and infinite horizons as $\phi_{2,fin}=\Always\lnot\mathcal{O}\land\Eventually(\mathcal{G}_{1}\land\Eventually(\mathcal{G}_{2}\land\Eventually(\mathcal{G}_{3}\land\Eventually\mathcal{G}_{init})))$ 
and 
$\phi_{2,inf}=\Always\lnot\mathcal{O}\land\Always\Eventually(\mathcal{G}_{1}\land\Eventually(\mathcal{G}_{2}\land\Eventually(\mathcal{G}_{3}))$, 
respectively, where $\phi_{2,fin}$ requires the vehicle to visit goal regions labeled as $\mathcal{G}_{1}, \mathcal{G}_{2}, \mathcal{G}_{3}$ and initial position sequentially, and $\phi_{1,inf}$ requires to visit initial positions and other goal regions infinitely often. Fig.~\ref{fig:Car_model} (a) and (b) show the learning curves of task $\phi_{2,fin}$ and $\phi_{2,inf}$, respectively, compared with different baselines. We can observe that our framework can be adopted in both DDPG and SAC to provide better performance  than other baselines in cluttered environments.

% We first consider the LTL task satisfaction in the form of $\phi$ in the environment of Fig.~\ref{fig:decomposition_example} referred as $Env1$.
% We implement all baselines with both DDPG and PPO, respectively. We run $10000$ episodes, each of which has maximum $1200$ steps. The learning results are shown in Fig.~\ref{fig:Goal_Reaching}. We observe that, "TL" with point-to-point sparse rewards and "NED" with minimizing the euclidean distance, have poor performances in complex environments, due to exploration issues.
% With the RRT* based reward, it is shown that our framework can be integrated with either the off-policy DPG or the on-policy DPG. The consecutive r-balls of rewards in the geometric space make the PPO perform better in Fig.~\ref{fig:Goal_Reaching}, since it allows to efficiently evaluate the current policies. 

\textbf{6.2 Quadrotor Model\text{ }} We implement our algorithms in a $3$D environment with Quadrotor\footnote{\url{https://github.com/Bharath2/Quadrotor-Simulation/tree/main/PathPlanning}} dynamics shown in Fig.~\ref{fig:Quadrator}, which shows the capability of handling cluttered environments and high dimensional systems. We also test two types of LTL specifications as $\phi_{1,fin}=\Always\lnot\mathcal{O}\land\Eventually\mathcal{G}_{1}\land\Eventually\mathcal{G}_{2}\land\Eventually\mathcal{G}_{3}$
 and $\phi_{1,inf}=\Always\lnot\mathcal{O}\land\Always\Eventually\mathcal{G}_{1}\land\Always\Eventually\mathcal{G}_{2}\land\Always\Eventually\mathcal{G}_{3}$. The learning results for these tasks are shown in Fig~\ref{fig:Quadrator} (a) and (b), respectively. %, which shows the state-of-art of our framework.

% As for general LTL tasks, we first show the decomposition process for the LTL task $\phi_{1,inf}=\Always\lnot\mathcal{O}\land\phi_{g_{1}}$ over infinite horizons, where $\phi_{g_{1}}=\Always\Eventually\mathcal{G}_{1}\land\Always\Eventually\mathcal{G}_{2}\land\Always\Eventually\mathcal{G}_{3}$ requires to infinitely visit goal regions labeled as $\mathcal{G}_{1}, \mathcal{G}_{2}, \mathcal{G}_{3}$. 
% The resulting truncated NBA and decomposed trajectories of TL-RRT* are shown in Fig.~\ref{fig:decomposition_example}, where decomposed sub-paths are expressed as $\mathcal{R}_{F}=\mathcal{R}_{red}(\mathcal{R}_{blue}\mathcal{R}_{pink} \mathcal{R}_{brown})^{\omega}$, such that the distributed DPGs are applied to train the optimal sub-policies for each one in parallel.

Then, we increase the complexity by random sampling $12$ obstacle-free goal regions in the $3$D environment and set the rich specifications as
$\phi_{3,fin}=\Always\lnot\mathcal{O}\land(\left(\Eventually\mathtt{\mathcal{G}_{1}}\land\Eventually\mathtt{\left(\mathcal{G}_{2}\land\Eventually\mathtt{\ldots\land \Eventually\mathcal{G}_{12}}\right)}\right)$, and
$\phi_{3,inf}=\Always\lnot\mathcal{O}\land\Always\Eventually\mathcal{G}_{1}\land\Always\Eventually\mathcal{G}_{2}\ldots\land\Always\Eventually\mathcal{G}_{12}$. The results are shown in Fig.~\ref{fig:Baselines} (a) and (b), and we observe that the "TL" baseline is sensitive to the environments and has poor performances, and when the optimal trajectories become complicated in the sense of the complexity of LTL tasks, "G-RRT*" easily converges to the sub-optimal solutions.

% For all tasks $\phi_{1,inf}, \phi_{2,inf}, \phi_{3,inf}$ over infinite horizons, we apply the DDPG to compare our framework with baselines "TL" and "G-RRT*". The results are shown in Fig.~\ref{fig:Baselines}, and we observe that the "TL" method is sensitive to the environments, and when the optimal trajectories become complicated in the sense of the complexity of LTL tasks, "G-RRT*" easily converge to the sub-optimal solutions. The learning results for these tasks are shown in Fig~\ref{fig:Quadrator} (a) and (b), respectively.

% The finite-horizon form of $\phi_{1,inf}$, $\phi_{2,inf}$, and $\phi_{3,inf}$ are expressed by removing the "Always" operator in the part $\phi_{g}$ e.g., $\phi_{1,f}=\Always\lnot\mathcal{O}\land\Eventually\mathcal{G}_{1}\land\Eventually\mathcal{G}_{2}\land\Eventually\mathcal{G}_{3}$ is the finite-horizon case of $\phi_{1,inf}$. It is shown that our framework generalizes the LTL satisfaction over both finite and infinite horizons. 

\begin{figure}[!t]
\begin{center}
\centerline{\includegraphics[width=\columnwidth]{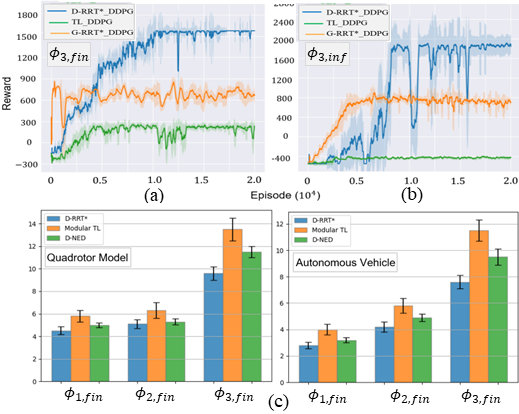}}
\caption{(a) (b) Results of baselines for more complex task $\phi_{3,fin}$ and $\phi_{3,inf}$ in cluttered $3$D environments; (c) Training time comparison for all tasks over finite horizons in both environments and dynamic systems.}
\label{fig:Baselines}
\end{center}
\end{figure}

\textbf{6.3 Performance Evaluation\text{ }}
Since our algorithm learns to complete the task faster, it terminates each episode during learning earlier for tasks over finite horizons. To illustrate the efficiency, we implement the training process $10$ times for tasks $\phi$, $\phi_{1,fin}$, $\phi_{2,fin}$, $\phi_{3,fin}$ in both cluttered environments and dynamics, and record the average training time compared with the baseline modular "TL"~\cite{Cai2021modular} and distributed "NED" (D-NED). The results in Fig.~\ref{fig:Baselines} (d) show that we have optimized the learning efficiency and the training time is reduced. In practice, we can apply distributed computing to train each local DPG of sub-tasks simultaneously for complicated global tasks to alleviate the training burden.

We statistically run $200$ trials of the learned policies for each sub-task and record the average success rates and training time of both models, i.e., autonomous vehicle and quadrotor. The results are shown in Table~\ref{tab:success_rate}. We see that in cluttered environments, success rates of other all baselines are $0$, and our method achieves success rates of near $100\%$. As a result, the effectiveness of the performance has been significantly improved under environmental challenges.

\begin{table}
	\caption{\label{tab:success_rate} Analysis of success rates and training time.}
	\centering{}\resizebox{0.45\textwidth}{!}{%%
		\begin{tabular}{|c|c|c|c|}
			\hline 
			LTL Task & Dynamic model & Baseline rate & Success rate\tabularnewline
			\hline 
% 			\multirow{2}{*}{$\phi$} & D-RRT* & $100\%$ & $2.25 $ \tabularnewline
% 			\cline{2-4} \cline{3-4} \cline{4-4}
% 			& TL, NED & $0\%$ & $3.0$\tabularnewline
% 			\hline 
			\multirow{2}{*}{$\phi_{1,fin}$} & \multirow{2}{*}{Quadrotor} & D-RRT* & $100\%$ \tabularnewline
			\cline{3-4} \cline{3-4} \cline{4-4}
			& & TL, G-RRT* &  $0\%$ \tabularnewline
			\hline 
			
			\multirow{2}{*}{$\phi_{2,fin}$} & \multirow{2}{*}{Vehicle} & D-RRT* & $100\%$ \tabularnewline
			\cline{3-4} \cline{3-4} \cline{4-4}
			& & TL, G-RRT* &  $0\%$ \tabularnewline
			\hline 
			
			\multirow{2}{*}{$\phi_{3,fin}$} & \multirow{2}{*}{Quadrotor} & D-RRT* & $100\%$ \tabularnewline
			\cline{3-4} \cline{3-4} \cline{4-4}
			& & TL, G-RRT* &  $0\%$ \tabularnewline
			\hline 
			
			\multirow{2}{*}{$\phi_{3,fin}$} & \multirow{2}{*}{Vehicle} & D-RRT* & $100\%$ \tabularnewline
			\cline{3-4} \cline{3-4} \cline{4-4}
			& & TL, G-RRT* &  $0\%$ \tabularnewline
			\hline 
			
			\multirow{2}{*}{$\phi_{1,inf}$} & \multirow{2}{*}{Quadrotor} & D-RRT* & $100\%$ \tabularnewline
			\cline{3-4} \cline{3-4} \cline{4-4}
			& & TL, G-RRT* &  $0\%$ \tabularnewline
			\hline 
			
			\multirow{2}{*}{$\phi_{2,inf}$} & \multirow{2}{*}{Vehicle} & D-RRT* & $100\%$ \tabularnewline
			\cline{3-4} \cline{3-4} \cline{4-4}
			& & TL, G-RRT* &  $0\%$ \tabularnewline
			\hline 
			
			\multirow{2}{*}{$\phi_{3,inf}$} & \multirow{2}{*}{Quadrotor} & D-RRT* & $100\%$ \tabularnewline
			\cline{3-4} \cline{3-4} \cline{4-4}
			& & TL, G-RRT* &  $0\%$ \tabularnewline
			\hline 
			
			\multirow{2}{*}{$\phi_{3,inf}$} & \multirow{2}{*}{Vehicle} & D-RRT* & $100\%$ \tabularnewline
			\cline{3-4} \cline{3-4} \cline{4-4}
			& & TL, G-RRT* &  $0\%$ \tabularnewline
			\hline 
			
	\end{tabular}}
\end{table}

\section{Conclusion\label{sec:conclusion}}

Applying DPG algorithms to cluttered environments produces vastly different behaviors and results in failure to complete complex tasks.
A persistent problem is the exploration phase of the learning process and the density of reward designs that limit its applications to real-world robotic systems. This paper provides a novel path-planning-based reward scheme to alleviate this problem, enabling significant improvement of reward performance and generating optimal policies satisfying complex specifications in cluttered environments. To facilitate rich high-level specification, we develop an optimal decomposition strategy for the global LTL task, allowing to train all sub-tasks in parallel and optimize the efficiency.
The main limitation of our approach is to generalize various environments from a distribution. Future works aim at shrinking the gap of sim-to-real. We will also consider safety-critical exploration during learning and investigate multi-agent systems.

\bibliographystyle{IEEEtran}
\bibliography{reference}

\newpage
\appendix
% \begin{appendices}

\subsection{Summary of Geometric RRT*\label{append:RRT*}}

Before discussing the algorithm in details, it is necessary to introduce few algorithmic primitives as follows:

{\em Random sampling:} The $Sample$ function provides independent, uniformly distributed random samples of states, from the geometric space $X$.

{\em Distance and cost:} The function $dist: X\times X\rightarrow [0,\infty)$ is the metric that returns the geometric Euclidean distance. The function $Cost: X\rightarrow [0,\infty)$ returns the length of the path from the initial state $x_{0}$ to the input state.

{\em Nearest neighbor:} Given a set of vertices $V$ in the tree $G$ and a state $x'$, the function $Nearest(V, x)$ generates the closest state $x\in V$ from which $x'$ can be reached with the lowest distance metric.

{\em Steering:} Given two states $x, x'$, the function $Steer$ returns a state $x_{new}$ such that $x_{new}$ lies on the geometric line connecting $x$ to $x'$, and its distance from $x$ is at most $\eta$, i.e., $dist(x, x')\leq\eta$, where $\eta$ is an user-specified parameter. In addition, the state $x_{new}$ must satisfy $dist(x_{new}, x')\leq dist(x, x')$. The function $Steer$ also returns the straight line $\sigma$ connecting $x$ to $x_{new}$.

{\em Collision check:} A function $CollisionFree(\sigma)$ that detects if a state trajectory $\sigma$ lies in
the obstacle-free portion of space $X$. $C(\sigma)$ is the distance of $\sigma$.

{\em Near nodes:} Given a set of vertices $V$ in the tree $G$ and a state $x'$, the function $Near(V, x')$ returns a set of states that are closer than a threshold cost to $x'$:
\\
\begin{equation*}
        Near(V, x')= \left\{x\in V:dist(x,x')\leq\kappa\left(\frac{\log n}{n}\right)^{1/d}\right\},
\end{equation*}
\\
where $n$ is the number of vertices in the tree, $d$ is the dimension of the geometric space, and $\kappa$ is a constant.

{\em Optimal Path:} Given two states $x, x'$ in $V$, the function $Path(x, x')$ returns a local optimal trajectory $\sigma$.

The Alg.~\ref{alg:RRT*} proceeds as follows. First, the graph $G$ is initialized with $V\leftarrow \{x_{0}\}$ and $E\leftarrow\emptyset$ (line~\ref{alg:rrt*:init}). Then a state $x_{rand}$ is sampled from $X$ of $Env$ (line~\ref{alg:rrt*:sample}), then, the nearest node $x_{nearest}$ in the tree is found (line~\ref{alg:rrt*:nearest}) and extended toward the sample, denoted by $x_{new}$, in addition to the straight line $\sigma_{new}$ connecting them (line~\ref{alg:rrt*:steer}). If line $\sigma_{new}$ is collision free (line~\ref{alg:rrt*:collision_free}), the algorithm iterates over all near neighbors of the state $x_{new}$ and finds the state $x_{min}$ that has the lowest cost to reach $x_{new}$ (lines~\ref{alg:rrt*:iterate}-~\ref{alg:rrt*:end_iterate}). Then the tree is updated with the new state (lines~\ref{alg:rrt*:add_node}-~\ref{alg:rrt*:add_edge}), and the algorithm rewires the near nodes, using Alg.~\ref{alg:rewire} (line~\ref{alg:rrt*:rewire}). Alg.~\ref{alg:rewire} iterates over the near neighbors of the new state $x_{new}$ and updates the parent of a near state $x_{near}$ to $x_{new}$ if the cost of reaching $x_{near}$ from $x_{new}$ is less than the current cost of reaching to $x_{near}$.

\begin{algorithm}[tb]
  \caption{Geometric RRT* ($(V,E)$, $N$)}
  \label{alg:RRT*}
\begin{algorithmic}[1] 
  \State {\bfseries Initialize:} $G=(V,E)$; $V\leftarrow \{x_{0}\}$, $E\leftarrow\emptyset$ \label{alg:rrt*:init}
  \For{$i=1,\ldots,N$ }
  \State $x_{rand}\leftarrow Sample$ \label{alg:rrt*:sample}
  \State $x_{nearest}\leftarrow Nearest(V$, $x_{rand})$ \label{alg:rrt*:nearest}
  \State $x_{new}, \sigma_{new}\leftarrow Steer(x_{nearest}$, $x_{rand})$ \label{alg:rrt*:steer}
  \If{$CollisonFree(\sigma_{new})$} \label{alg:rrt*:collision_free}
  \State $X_{near}\leftarrow Near(V, x_{new})$ \label{alg:rrt*:iterate}
  \State $c_{min}\leftarrow\infty$, $x_{min}\leftarrow\text{ NULL}$, $\sigma_{min}\leftarrow\text{ NULL}$
  \For{$x_{near}\in X_{near}$}
  \State $\sigma\leftarrow Path(x_{near}, x_{new})$
  \If {$Cost(x_{near})+C(z)<c_{min}$}
  \State $c_{min}\leftarrow Cost(x_{near})+Cost(\sigma)$
  \State $x_{min}\leftarrow x_{near}$; $\sigma_{min}\leftarrow \sigma$
  \EndIf
  \EndFor \label{alg:rrt*:end_iterate}
  \State $V\leftarrow V\cup\left\{x_{new}\right\}$ \label{alg:rrt*:add_node}
  \State $E\leftarrow E\cup\left\{(x_{min}, x_{new})\right\}$ \label{alg:rrt*:add_edge}
  \State $(V,E)\leftarrow$ Rewire$((V,E), X_{near}, x_{new})$ \label{alg:rrt*:rewire}
  \EndIf
  \EndFor
  \State \textbf{return} $G=(V,E)$
\end{algorithmic}
\end{algorithm}

\begin{algorithm}[tb]
  \caption{Rewire($(V,E)$, $X_{near}$, $x_{new}$)}
  \label{alg:rewire}
\begin{algorithmic}[1]
  \For{$x_{near}\in X_{near}$}
  \State $\sigma\leftarrow Path(x_{new}, x_{near})$
  \If {$Cost(x_{new})+C(\sigma)<Cost(x_{near})$}
  \If {$CollisonFree(x)$}
  \State $x_{parent}\leftarrow Parent(x_{near})$
  \State $E\leftarrow E\setminus\left\{(x_{parent}, x_{near})\right\}$
  \State $E\leftarrow E\cup\left\{(x_{new}, x_{near})\right\}$
  \EndIf
  \EndIf
  \EndFor
  \State \textbf{return} $G=(V,E)$
\end{algorithmic}
\end{algorithm}

\begin{figure*}
\vskip 0.1in
\begin{center}
\centerline{\includegraphics[scale=0.5 ]{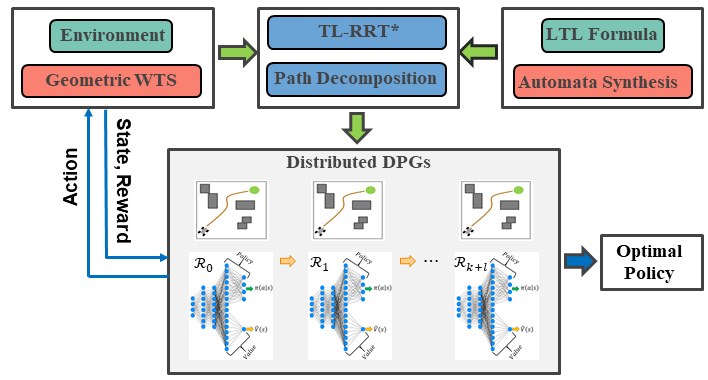}}
\caption{General diagram of the LTL-RRT*-Distributed DPGs method explained in Alg.~\ref{alg:framework}.}
\label{fig:diagram}
\end{center}
\vspace{-0.8cm}
\end{figure*}

\subsection{Analysis of Product State\label{append:product}}
In this section, we show that the product state can be applied with RL algorithms and the~\eqref{eq:reward_function} is a Markovian reward for the state $s^{\times}=(s_{t}, i_{t})$. Let $I=\left\{0,1\ldots,N_{p} \right\}$ be a set of all sequential indexes of the r-balls for an optimal trajectory $\boldsymbol{x}^*$, and define $f_{I}:\mathbb{R}\rightarrow I$ as function to return the index at each time $t$ s.t. $f_{I}(t)=i_{t}$, where the output of the function $f_{I}$ follows the requirements:
\begin{equation*}
\left\{ \begin{array}{cc}
i_{t}=0, \text{ if none of the r-balls are visited up to } t,\\
x_{i_{t}}=Proj(s_{i_{t}}), \text{ where } s_{i_{t}}=\underset{s\in \boldsymbol{s}_{t}}{\arg\min}\left\{D(Proj(s))\right\}.\\
\end{array}\right.
\end{equation*}
Consequently, given an input $t$, the $f_{I}$ generates a deterministic output such that we regard the tuple $(I, f_{I}, i_{0})$ as a deterministic automaton without accepting states~\cite{baier2008}, where $i_{0}=0$. The state $s^{\times}_{t}$ is derived from a product structure defined as:

\begin{defn}
The product between cl-MPD $\mathcal{M}=(S, S_{0}, A,p_{S},\AP,L, R, \gamma)$ and $(I, f_{I}, i_{0})$ is a tuple $\mathcal{M}^{\times}=(S^{\times}, S^{\times}_{0}, A ,p^{\times}_{S},\AP, L^{\times}, R, \gamma)$, where $S^{\times}=S\times I$ is a set of product states; $S^{\times}_{0}=S_{0}\times i_{0}$ is a set of initial states;
$p^{\times}_{S}$ is the transition distribution s.t. for states $(s^{\times}=(s,i), (s'^{\times}=(s',i')$, $p^{\times}_{S}\left(s'^{\times}|s^{\times},a\right)>0$ iff $p_{S}\left(s'|s,a\right)>0, a\in A$ and $i'=f_{I}(i)$; $L^{\times}(s^{\times}=L(s), \forall s^{\times}=(s,i)$.
\end{defn}
According to~\cite{baier2008}, $\mathcal{M}^{\times}$ is a product MDP structure that allows using RL methods to find the optimal policy, where the state $i_{t}$ of $s^{\times}_{t}$ tracks the history of state trajectory $\boldsymbol{s}_{t-1}$ up to $t$. Consequently, the reward~\eqref{eq:reward_function} is a Markovian reward function.

\subsection{Theorem~\ref{thm:RRT*} Proof \label{proof:thm1}}

First, we introduce two relevant definitions:
\begin{defn}
\cite{baier2008}
A Markov
chain $MC_{\mathcal{P}}^{\pi}$ of the $\mathcal{M}$ is a sub-MDP of $\mathcal{M}$ induced by a policy $\pi$.
\end{defn}

\begin{defn}\label{def:induced_markov_chain}
\cite{Durrett1999} States of any Markov
chain $MC_{\mathcal{P}}^{\pi}$ under policy $\pi$ are represented by a disjoint union of a transient
class $T(\pi)$ and $n_R$ closed
irreducible recurrent classes $Re^{j}(\pi)$,
$j\in\left\{ 1,\ldots,n_{Re}\right\} $, where a class is a set of states. That is, for any policy $\pi$, one has
\begin{equation*}
    MC_{\mathcal{P}}^{\pi}=Tr(\pi)\sqcup Re^{1}(\pi)\sqcup Re^{2}(\pi)\sqcup\ldots Re^{n_{Re}(\pi)}.
\end{equation*}
\end{defn}

As discussed in \cite{Durrett1999}, for each state in recurrent class,
it holds that $\stackrel[n=0]{\infty}{\sum}p^{n}\left(s^{\times},s^{\times}\right)=\infty$,
where $s^{\times}\in\ensuremath{\ensuremath{R_{\pi}^{j}}}\cap F_{k}^{\mathcal{\mathcal{P}}}$
and $p^{n}\left(s^{\times},s^{\times}\right)$ denotes the probability of returning
from a transient state $s^{\times}$ to itself in $n$ steps. This means that
each state in the recurrent class occurs infinitely often.

Then, we prove it by the contradiction. Suppose we have a policy $\bar{\pi}$ that is optimal and does not satisfy go-reach task $\phi_{p}$, which means the robot derived by $\bar{\pi}$ will not reach the goal station. According to the reward design~\eqref{eq:reward_function}, the robot is only assigned repetitive reward $R_{++}$ when it reaches and stays at the destination s.t. $D(Proj(s_{t}))=0$. 
We have that rewards of states in recurrent classes are equal zero i.e., $R(s^{\times}_{t})= 0, \forall  s^{\times}_{t}\in Re^{j}(\bar{\pi}), \forall j\in\left\{ 1,\ldots,n_{Re}\right\}$. Recall that we have $N_{p}$ r-balls, and the best case of $\bar{\pi}$ is to consecutively pass all these balls sequentially without reaching destination. We obtain the upper-bound of $J(\bar{\pi})$ as:
\begin{equation}
    J(\bar{\pi})< R_{+}\cdot\frac{1-\gamma^{N_{p}}}{1-\gamma}\label{eq:upper}
\end{equation}
Per assumption 1, we can find another policy $\pi^{*}$ that reaches and stays at the destination s.t. $R(s'^{\times}_{t})=R_{++}, \forall  s'^{\times}\in Re^{j}(\pi^{*}), \forall j\in\left\{ 1,\ldots,n_{Re}\right\}$. The worst case of $\pi^{*}$ is to pass no r-balls in the class $Tr(\bar{\pi})$ and only reach the goal stations. We obtain the lower-bound of $J(\pi^{*})$ as:
\begin{equation}
    J(\pi^{*})\geq \underline{M}R_{++}\cdot\frac{1}{1-\gamma}\label{eq:lower},
\end{equation}
where $\underline{M}=\gamma^{\bar{n}}$, and $\bar{n}$ is maximum number of steps reaching the goal station region. Consequently, for~\eqref{eq:upper} and~\eqref{eq:lower}, if we select
\begin{equation}
    R_{++}\geq \frac{R_{+}\cdot(1-\gamma^{N_{p}})}{\underline{M}},
\end{equation}
we guarantee that $J(\pi^{*})>J(\bar{\pi})$, which
contradicts the fact that $\bar{\pi}$ is an optimal policy. This concludes the theorem. 

Note that, the above proof only considers the worst case of optimal policies satisfying goal-reaching task $\phi_{p}$. According to the design of the RRT* reward~\eqref{alg:RRT*}, in practice, we can achieve better convergence due to the high density of r-balls in geometric space.

\subsection{Diagram for Alg.~\ref{alg:framework}\label{append:diagram}}

Here we present a general diagram of our proposed method in Alg.~\ref{alg:framework}. From the given environment $Env$ with geometric space $X$, the G-WTS $\mathcal{T}$ is constructed. This transition system, together with the NBA $\mathcal{B}$ generated from the LTL task $\phi_{p2}$, are used to construct the PBA $P$. The TL-RRT* method is applied over $P$ to compute the optimal accepting run $\tau^*_F$ and the optimal trajectory $R^*_F$. Using the path decomposition with respect to the order of segments in \eqref{eq:lasso}, distributed DPGs are trained in parallel over the episodes, and the resulting optimal distributed policies are concatenated sequentially to satisfy the LTL formula in the form of $\phi_{p2}$.
Based on the decomposition properties, we have: 

\begin{lem}
\label{lem:D-RRT*}
If Assumption 1 holds,  by selecting $R_{++}$ to be sufficiently larger than $R_{+}$, i.e., $R_{++} \gg R_{+}$, Alg.~\ref{alg:framework} using a suitable DPG algorithm can generate the optimal policy $\pi^{*}_{\theta}=(\pi^{*}_{0}\pi^{*}_{i}\ldots \pi^{*}_{K})(\pi^{*}_{K+1}\ldots \pi^{*}_{K+l})^{\omega}$ satisfying the general LTL task $\phi_{p2}$, i.e., $\Pr{}_{M}^{\pi^{*}}(\phi_{p2})>0$
in the limit.
\end{lem}

\begin{proof}
Since the TL-RRT*~\cite{luo2021abstraction} has shown to find the optimal geometric path to satisfy $\phi_{p2}$, which is decomposed into sub-tasks in the form of $\mathcal{R}^{*}_{F}=(\mathcal{R}_{0}\mathcal{R}_{i}\ldots \mathcal{R}_{K})(\mathcal{R}_{K+1}\ldots \mathcal{R}_{K+l})^{\omega}$, from Theorem~\ref{thm:DPG_RRT*}, each optimal sub-policy $\pi^{*}_{i}$ achieves the sub-task of $\mathcal{R}_{i}$. This concludes Lemma~\ref{lem:D-RRT*}.
Note that the LTL formula $\phi$ is a special case of $\phi_{p2}$ that only has one DPG for training and is solved via Alg.~\ref{alg:framework}. 
\end{proof}

%%%%%%%%%%%%%%%%%%%%%%%%%%%%%%%%%%%%%%%%%%%%%%%%%%%%%%%%%%%%%%%%%%%%%%%%%%%%%%%
%%%%%%%%%%%%%%%%%%%%%%%%%%%%%%%%%%%%%%%%%%%%%%%%%%%%%%%%%%%%%%%%%%%%%%%%%%%%%%%

\subsection{Experimental Setting\label{append:experiment}}

All experiments are conducted on a 16GB computer using $1$ Nvidia RTX 3060 GPU.
In each experiment, the LTL tasks are converted into NBA using the tool:~\url{http://www.lsv.fr/~gastin/ltl2ba/}. The cl-MDP between the dynamic system and environmental geometric space is synthesized on-the-fly.
The parameters of the reward scheme are set as $R_{-}\in\left\{-100, -200\right\}$, $R_{++}=200$, $R_{+}=5$.

We run $10000$ episodes for the task $\phi_{2,inf}, \phi_{2,fin}$,  $20000$ episodes for the tasks $\phi_{1,inf}, \phi_{1,fin}, \phi_{3,inf}, \phi_{3,fin}$. 
Every episode has maximum $1000$ steps for tasks $\phi_{1,inf}, \phi_{1,fin}, \phi_{2,inf}, \phi_{2,fin}$, and each episode of the tasks $\phi_{3,inf}, \phi_{3,f}$ has maximum $1500$ steps.

As for each actor/critic structure, we use the same feed-forward neural network setting with 3 fully connected layers with $[64, 64, 64]$ units and the ReLU activation function. We use the implementations of DDPG and PPO for tuning parameters according to the OpenAI baselines:~\url{https://github.com/openai/baselines}. 
\end{document}